\documentclass[journal,comsoc]{IEEEtran}

\newcommand{\normF}[1]{\lVert#1\rVert_F}
\usepackage{amsmath,amsfonts,amssymb,graphicx}
\usepackage{hyperref}
\usepackage{color}
\usepackage{xcolor}
\usepackage{mathrsfs}
\usepackage[utf8]{inputenc}
\usepackage[english]{babel} 
\usepackage{tikz}
\usetikzlibrary{arrows}
\usetikzlibrary{positioning,fit,calc}
\tikzset{block/.style={draw,thick,text width=1.5cm,minimum height=1cm,align=center},
	line/.style={-latex} }
\usepackage{lipsum,adjustbox}
\usepackage{bookmark}
\usepackage{multirow}

\usepackage{array}
\newcolumntype{M}{>{\centering\arraybackslash}m{1cm}}
\usepackage{hyperref}
\usepackage{color}
\usepackage{xcolor}
\usepackage{mathrsfs}
\usepackage[utf8]{inputenc}
\usepackage[english]{babel}
\usepackage{algorithm}
\usepackage{algpseudocode}% http://ctan.org/pkg/algorithmicx
\usepackage{setspace}
\usepackage{bookmark}
\DeclareGraphicsExtensions{.jpg,.pdf,.png}
\RequirePackage{amsopn}
\RequirePackage{amsfonts}
\RequirePackage{amsthm}
\theoremstyle{definition}

\newtheorem{prop}{Proposition}
\RequirePackage{amsthm}
\usepackage{amssymb}
\ifCLASSOPTIONcompsoc
  \usepackage[nocompress]{cite}
\else
  \usepackage{cite}
\fi
\ifCLASSINFOpdf
\else
\fi
\begin{document}
\title{A Novel Approach to Quantized Matrix Completion Using Huber Loss Measure}
\author{Ashkan~Esmaeili and
        Farokh~Marvasti
\thanks{A. Esmaeili was with the Department
of Electrical Engineering, Stanford University, California, USA~e-mail:{~esmaeili.ashkan@alumni.stanford.edu}}% <-this % stops a space
\thanks{A. Esmaeili and F. Marvasti are now with the Electrical Engineering Department and Advanced Communications Research Institute (ACRI), Sharif University of Technology, Tehran, Iran.}% <-this % stops a space
}
\markboth{}%
{Esmaeili \MakeLowercase{\textit{et al.}}: Bare Demo of IEEEtran.cls for IEEE Communications Society Journals}
\maketitle
\begin{abstract}
\quad In this paper, we introduce a novel and robust approach to Quantized Matrix Completion (QMC). First, we propose a rank minimization problem with constraints induced by quantization bounds. Next, we form an unconstrained optimization problem by regularizing the rank function with Huber loss. 
Huber loss is leveraged to control the violation from quantization bounds due to two properties: 1- It is differentiable, 2- It is less sensitive to outliers than the quadratic loss.
A Smooth Rank Approximation is utilized to endorse lower rank on the genuine data matrix. Thus, an unconstrained optimization problem with differentiable objective function is obtained allowing us to advantage from Gradient Descent (GD) technique. Novel and firm theoretical analysis on problem model and convergence of our algorithm to the global solution are provided. Another contribution of our work is that our method does not require projections or initial rank estimation unlike the state-of-the-art. In the Numerical Experiments Section, the noticeable outperformance of our proposed method in learning accuracy and computational complexity compared to those of the state-of-the-art literature methods is illustrated as the main contribution.
\end{abstract}
%%%%%%%%%%%%%%%%%%%%%%%%%%
\section{Introduction}\label{Intro}
In this paper, we extend the Matrix Completion (MC) problem, which has been considered by many authors in the past decade \cite{candes2009exact,cai2010singular, keshavan2010matrix}, to the Quantized Matrix Completion (QMC) problem. In QMC, accessible entries are quantized rather than continuous, and the rest are missing. The purpose is to recover the original continuous-valued matrix under certain assumptions. \\
QMC problem addresses wide variety of applications including but not limited to collaborative filtering \cite{koren2009matrix}, sensor networks \cite{biswas2004semidefinite}, learning and content analysis \cite{esmaeili2018transduction} according to \cite{bhaskar}.
\par
A special case of QMC, one-bit MC, is considered by several authors. In \cite{davenport20141} for instance, a convex programming is proposed to recover the data by maximizing a log-likelihood function. In \cite{cai2013max}, a maximum likelihood (ML) set-up is proposed with max-norm constraint towards one-bit MC. In \cite{ni2016optimal}, a greedy algorithm as an extension of conditional gradient descent, is proposed to solve an ML problem with rank constraint. However, the scope of this paper is not confined to one-bit MC, and covers multi-level QMC. We investigate multi-level QMC methodologies in the literature hereunder:\\ In 
\cite{lan2014matrix}, the robust \textbf{Q-MC} method is introduced based on projected gradient (PG) approach in order to optimize a constrained log-likelihood problem. The projection guarantees shrinkage in Trace norm tuned by a regularization parameter.\\
Novel QMC algorithms are introduced in \cite{bhaskar}. An ML estimation under an exact rank constraint is considered as one part. Next, the log-likelihood term is penalized with log-barrier function, and bilinear factorization is utilized along the Gradient Descent (GD) technique to optimize the resulted unconstrained problem. The suggested methodologies in \cite{bhaskar} are robust, leading to noticeable accuracy in QMC. However, the two algorithms in \cite{bhaskar} depend on knowledge of an upper bound for the rank (an initial rank estimation), and may suffer from local minimia or saddle points issues.
%In \cite{esmaeili2018recovering}, another QMC approach based on Augmented Lagrangian Function and Biliniear Factorization is analyzed. The performance in recovery is enhanced in comparison to previous works. The method in \cite{esmaeili2018ecovering} does not require an initial firm rank estimation. Appropriate choice of regularization parameters ensures the recovery under certain assumptions.\\
% We will elaborate that the proposed method in this paper is not dependent on any initial rank estimation, and appropriate parameter selection guarantees globally optimal solution.
In \cite{esmaeili2018recovering}, Augmented Lagrangian method (ALM) and bilinear factorization are utilized to address the QMC. Enhanced accuracy in recovery is observed compared to previous works in \cite{esmaeili2018recovering}.
\par
In this paper, Huber loss and Smoothed Rank Function (SRF), which are differentiable, are utilized to induce penalty for violating quantization bounds and increase in the rank, respectively. Differentiability makes the optimization framework suitable for GD approach. It is worth noting that although Huber is convex, SRF is generally non-convex. However, we leverage Graduated Non-Convexity (GNC) approach to solve consecutive problems, in which the local convexity in a specific domain enclosing the global optimum is maintained. The solution to each problem is utilized as a warm-start to the next problem. Utilizing warm-starts and smooth transition between problems ensure the warm-start falls in a locally convex enclosure of the global optimum. It is theoretically analyzed how the SRF parameter can be tuned and shrunk gradually to guarantee the local convexity in each problem is obtained which finally leads to the global optimum.
 Unlike \cite{bhaskar}, our method does not require an initial rank upper bound estimation, neither projections as in \cite{bhaskar} and \cite{lan2014matrix}.
\par
The rest of the paper is organized as follows: Section \ref{pm} includes the problem model and discussion on Huber Loss. In Section \ref{SRF}, Smoothing Rank Approximation is discussed. Section \ref{proposed}, includes our proposed algorithm. Theoretical analysis for the global convergence of our algorithm is given in Section \ref{TA}.
Simulation results are provided in Section \ref{NE}. Finally, the paper is concluded in Section \ref{conc}.
\section{Problem Model}\label{pm}
We assume a quantized matrix $\mathcal{M} \in \mathcal{R}^{m \times n}$ is partially observed; i.e., the entries of $\mathcal{M}$ are either missing or reported as integer values (levels) $m_{ij}$. Different levels are spaced with distance $g$ known as the quantization gap. We also assume the rounding rule forces the entries of the original matrix to be quantized to a level within $\pm \frac{g}{2}$ of their vicinities.
We assume the original matrix, from which the quantized data are obtained, has the low-rank property as in many practical settings. Thus, the following optimization problem on $\boldsymbol{X} \in \mathcal{R}^{m \times n}$ is reached:
\begin{equation}\label{p0}
    \begin{aligned}
        &\underset{\boldsymbol{X}}{\text{min}}&&\textrm{rank}\boldsymbol{X}\\
        &\text{subject to} && l_{ij} \leq x_{ij} \leq u_{ij} 
        &&& \forall (i,j) \in \Omega
    \end{aligned},
\end{equation}\\
%As known, convex relaxation
%\begin{equation}\label{p1}
%    \begin{aligned}
%        &\underset{\boldsymbol{X}}{\text{min}}&&\norm{\boldsymbol{X}}\\
%        &\text{subject to} && l_{ij} \leq x_{ij} \leq u_{ij} 
%        &&& \forall (i,j) \in \Omega
%    \end{aligned},
%\end{equation}
where $\Omega$ is the observation set, $u_{ij}$ and $l_{ij}$ are the upper and lower quantization bounds of the $ij$-th observed entry. In our model, the bounds are assumed to symmetrically enclose $x_{ij}$, i.e., $l_{ij}=m_{ij}-\frac{g}{2} \leq x_{ij} \leq m_{ij}+\frac{g}{2}=u_{ij}$. We also add that the number of levels is considered to be known and no entry exceeds the quantization bounds of ultimate levels in the original matrix.\\ 
The Huber function for the $ij$-th entry in $\Omega$ is defined as follows:
  \[
    \mathcal{H}_{ij}(x_{ij})=\left\{
                \begin{array}{ll}
                (x_{ij}-m_{ij})^2, \quad \quad \quad |x_{ij}-m_{ij}| \leq \frac{g}{2}\\
                  g(|x_{ij}-m_{ij}|-\frac{1}{4}g), \quad  o.w.
                \end{array}
              \right.
  \]
We modify the Huber loss by subtracting $\frac{1}{2}g^2$; i.e., $$\tilde{\mathcal{H}}_{ij}(x_{ij})=\mathcal{H}_{ij}(x_{ij})-\frac{1}{4}g^2$$
%Huber function, which is differentiable and convex, penalizes the small residuals with squared loss, and grows linearly to be less sensitive to outliers for larger residual values. Huber function induces small $l_2$-norm penalty on small residuals but aggresive linearly growing penalty on violations from the interval $2\delta$. In our model, we leverage the Quasi-Huber function $\mathcal{H_Q}_{ij}(.)$ defined as follows:
%  \[
%    \mathcal{H_Q}_{ij}(x_{ij})=\left\{
%                \begin{array}{ll}
%                (|x_{ij}-m_{ij}|-\frac{g}{2})^2, \quad \quad |x_{ij}-m_{ij}| \geq \frac{g}{2}\\
%                  0,\quad \quad \quad \quad \quad \quad  \quad \quad \quad|x_{ij}-m_{ij}| \leq \frac{g}{2}
%                \end{array}
%              \right.
%  \]
Huber loss is used in robust regression to advantage from desirable properties of both $l_2$ and $l_1$ penalty. Noise may have forced the original matrix entries to deviate from their genuine quantization bounds. Thus, we intentionally use linearly growing ($l_1$) penalty for violations from quantization bounds to be less sensitive to outliers as squared loss is. The interpretation of translating Huber as defined above is to reward entries which hold in constraints of the problem \ref{p0}. This reward is quadratic which does not vary as sharply as $l_1$ penalty on the feasible region. The entire feasible region is delighted. Thus, sharply varying behavior on feasible region is pointless. In addition, this specific quadratic reward makes the compromise of $l_1$ and $l_2$ convex and differentiable to profit us later in the paper.
The motivation of applying Huber function in our problem is to turn the constrained problem \ref{p0} to an unconstrained regularized problem. The regularization term is defined as follows:
  \begin{equation}
    \begin{aligned}
    \mathcal{H}_\Omega(\boldsymbol{X})=\sum_{(i,j)\in \Omega} \tilde{\mathcal{H}}_{ij}(x_{ij})
    \end{aligned}
\end{equation}
Thus, the unconstrained problem can be written as:
\begin{equation}\label{p2}
    \begin{aligned}
        &\underset{\boldsymbol{X}}{\text{min}}&&G(\boldsymbol{X},\lambda):= \text{rank}\boldsymbol{X}+ \lambda\mathcal{H}_{\Omega}(\boldsymbol{X})
    \end{aligned}
\end{equation}
\textbf{Assumption1.} \textit{The set of global solutions to the problem \ref{p0} is a singleton; i.e., problem \ref{p0} has a unique global minimizer $\boldsymbol{X}^*$}.\label{ass}\\
Let $\text{rank}(\boldsymbol{X}^*)=r^*$, $S(\lambda)$ denote the set of global minimizers of the problem \ref{p2},  $\mathcal{B}_1=\{ \boldsymbol{X} | \text{rank}\boldsymbol{X}<r^* \}$, $\mathcal{B}_2=\{\boldsymbol{X}| \text{rank}\boldsymbol{X}=r^*,~\exists (i,j) | \tilde{\mathcal{H}}({x_{ij}})>0\}$, and $\Delta_1, \Delta_2$ be defined as follows:
\begin{equation}
\Delta_1 = \underset{\boldsymbol{x} \in \mathcal{B}_1}{\min} ~~ \underset{(i,j) \in \Omega}{\max} ~ \big \{\tilde{\mathcal{H}}(x_{ij})|~\tilde{\mathcal{H}}(x_{ij})>0 \big \}.
\end{equation}
\begin{equation}
\Delta_2 = \underset{\boldsymbol{x} \in \mathcal{B}_2}{\min} ~~ \underset{(i,j) \in \Omega}{\max} ~ \big \{\tilde{\mathcal{H}}(x_{ij})|~\tilde{\mathcal{H}}(x_{ij})>0 \big \}.
\end{equation}
$\Delta_1$ is trivially greater than zero. Otherwise, a feasible solution to problem\ref{p0} exists with rank smaller than $r^*$ which is contradictory to the assumption $\text{rank}{X^*}=r^*$. Let $\Delta=\min\{\Delta_1,\Delta_2\}$.
We add two assumptions to follow our line of proof. (These assumptions address worst-case scenarios. In practice, they are not required to be such tight):\\
\textbf{Assumption 2.}\label{ass2}
$\Delta > \frac{(|\Omega|-1)g^2}{4}$\\
\textbf{Assumption 3.}\label{ass3}
$\frac{r^*}{\Delta-\frac{(|\Omega|-1)g^2}{4}} \leq  \frac{4}{g^2|\Omega|+\epsilon}$, where $\epsilon$ is any positive small constant.
\begin{prop}\label{prop}
\textit{Suppose $\epsilon$ is any positive small constant. For each $\lambda$ which holds in $\frac{r^*}{\Delta-\frac{(|\Omega|-1)g^2}{4}} \leq \lambda \leq \frac{4}{g^2|\Omega|+\epsilon}$, $S(\lambda)$ is the singleton $\{\boldsymbol{X}^*\}$.} 
\end{prop}
\begin{proof}
Suppose $\tilde{\boldsymbol{X}} \in S(\lambda)$.
Three cases can be considered for $\text{rank}{\tilde{\boldsymbol{X}}}$:\\
\textbf{Case} $\mathcal{I}$: $\text{rank}{\tilde{\boldsymbol{X}}} > r^*$. We have: \\
\begin{equation}
\begin{aligned}
G(\tilde{\boldsymbol{X}},\lambda)=\text{rank}{\tilde{\boldsymbol{X}}}+\lambda \sum_{(i,j)\in \Omega} \tilde{\mathcal{H}}_{ij}(x_{ij})\geq r^*+1-\lambda\frac{|\Omega| g^2}{4} \geq \\
r^*+1-\frac{|\Omega|g^2}{|\Omega| g^2+\epsilon}>r^* \geq \text{rank}(\boldsymbol{X}^*)+\lambda H_{\Omega}(\boldsymbol{X}^*)=G(\boldsymbol{X}^*,\lambda)\\
\end{aligned}
\end{equation}
$\Rightarrow G(\tilde{\boldsymbol{X}},\lambda) > G(\boldsymbol{X}^*,\lambda)$, which is in contradiction to the assumption that $\tilde{\boldsymbol{X}}$ is the global minimizer of problem \ref{p2}. We used the definition of $G$, translated Huber minimal value, the upper bound on $\lambda$ in Assumption 3, and the fact that $H_{\Omega}(\boldsymbol{X}^*)$ is negative due to feasibility of $\boldsymbol{X}^*$ in problem \ref{p0}.\\
\textbf{Case} $\mathcal{II}: \text{rank}{\tilde{\boldsymbol{X}}} < r^*$; i.e.,  
$\tilde{\boldsymbol{X}} \in \mathcal{B}_1$.\\ Thus, using lower bound on $\lambda$ in Assumption 3:
\begin{equation*}
G(\tilde{\boldsymbol{X}},\lambda)=\text{rank}{\tilde{\boldsymbol{X}}}+\lambda \sum_{(i,j)\in \Omega} \tilde{\mathcal{H}}_{ij}(x_{ij})\geq
\end{equation*}
\begin{equation*}
\text{rank}{\tilde{\boldsymbol{X}}}+\lambda \Delta-\lambda\frac{(|\Omega|-1)g^2}{4} \geq \text{rank}{\tilde{\boldsymbol{X}}}+r^* > r^* \geq
\end{equation*}
\begin{equation}\label{eq7}
\begin{aligned}
\text{rank}{{\boldsymbol{X}^*}} +\lambda\mathcal{H}_\Omega({\boldsymbol{X}^*})= G({\boldsymbol{X}^*},\lambda)
 \Rightarrow  G(\boldsymbol{X}^*,\lambda)<G(\tilde{\boldsymbol{X}},\lambda)
	\end{aligned}
\end{equation}
which is again in contradiction with the assumption that $\tilde{\boldsymbol{X}}$ is the global minimizer of problem \ref{p2}. \\
\textbf{Case} $\mathcal{III}: \text{rank} {\tilde{\boldsymbol{X}}} = r^*$\\
If at least one entry of $\tilde{\boldsymbol{X}}$ violates the constraints in problem \ref{p0}, then $\tilde{\boldsymbol{X}} \in \mathcal{B}_2$, and similar reasoning in \eqref{eq7} can be applied to contradict the global optimality of $\boldsymbol{X}^*$. 
Finally, if $\text{rank}{\tilde{\boldsymbol{X}}}=r^*$ and $\tilde{\boldsymbol{X}}$ holds in the constraints in problem \ref{p0}, then $\tilde{\boldsymbol{X}}=\boldsymbol{X}^*$ by Assumption 1. Thus, $S(\lambda)=\{\boldsymbol{X}^*\}$.
\end{proof}
\section{Smoothed Rank Approximation}\label{SRF}
While reviewing Huber loss, we mentioned it is convex and differentiable. 
%Thus, a gradient look-up table is available for Huber function. 
We aim to find a convex differentiable surrogate for rank function to leverage GD method. Trace norm is usually considered as the rank convex surrogate. However, Trace norm is not differentiable. In addition, Sub-Gradient methods for Trace norm are computationally complex and have convergence rate issues. Thus, we seek for a convex differentiable rank approximation to leverage GD instead of Sub-Gradient based approaches. In this regard, we approximate the rank function in problem \ref{p0} with the Smoothed Rank Function (SRF). SRF is defined using a certain function satisfying QRA conditions introduced in \cite{malek2014recovery}.
Assume $f_\delta(x)$ satisfies QRA conditions, and let $f_{\delta}(x)=f(\frac{x}{\delta})$.
Among functions satisfying the QRA conditions, we consider $f(x)=e^{-\frac{x^2}{2}}$ throughout this paper. Let $\sigma_i(\boldsymbol{X})$ denote the $i-$th singular value of $\boldsymbol{X}$. We define $F_\delta(\boldsymbol{X})$ as follows: \begin{equation}\label{Fdelta}
F_{\delta}(\boldsymbol{X})= \sum_{i=1}^n f_{\delta}(\sigma_i(\boldsymbol{X})).
\end{equation} \\
%Thus, for a matrix $\boldsymbol{X}$ we have:
Our proposed SRF is considered to be $n-F_\delta(\boldsymbol{X})$. It can be observed that $f_{\delta}(x)$ converges in a pointwise fashion to the Kronecker delta function as $\delta\to 0$. Thus, we have:
\begin{align}\label{p9}
\nonumber \lim_{\delta\to 0}~[n - F_\delta(\boldsymbol{X})] &=
\lim_{\delta\to 0}~[n - \sum_{i=1}^n f_{\delta}(\sigma_i(\boldsymbol{X}))] \\  =n-\sum_{i=1}^n\delta_0(\sigma_i(\boldsymbol{X}))
& = \text{rank}(\boldsymbol{X}).
\end{align}
Therefore, when $\delta \to 0$, SRF directly approximates the $\text{rank}$ function. As a result, we substitute the rank function in problem \ref{p2} with the proposed SRF as follows:
\begin{equation}\label{p3}
    \begin{aligned}
        &\underset{\boldsymbol{X}}{\text{min}}&&\tilde{G}_\delta(\boldsymbol{X},\lambda):= n-F_\delta(\boldsymbol{X})+ \lambda\mathcal{H}_\Omega(\boldsymbol{X})
    \end{aligned}
\end{equation}
The advantage of SRF to the
rank function is that $F_\delta$ is smooth and differentiable. Hence, GD can be utilized for minimization. However, the SRF is in general non-convex. When $\delta$ tends to $0$, the SRF is a good rank approximation as shown in \ref{p9} but with many local minimia. In order for GD not to get trapped by local minimia, we start with large $\delta$ for SRF. When $\delta \to \infty$ the SRF becomes convex (proved in Proposition \ref{prop2}) yielding a unique global minimizer for problem \ref{p3}. Yet, SRF with large $\delta$ is a bad rank approximation. This is where GNC approach as introduced in \cite{blake1987visual} is leveraged; i.e., we gradually decrease $\delta$ to enhance accuracy of rank approximation. A sequence of problems as in \ref{p3} (one for each value of $\delta$) is obtained. The solution to problem with a fixed $\delta$ is used as a warm-start for the next problem with new $\delta$. If $\delta$ is shrunk gradually, then the continuity property of $f_\delta$ (a QRA condition) leads to close solutions for subsequent problems. This way, GD is less probable to get trapped in local minima. The Huber loss which is also convex, acts like the augmented term in augmented Lagrangian method. It helps making the Hessian of $\tilde{G}_\delta$ locally positive-definite. Choosing warm-starts to fall in a convex vicinity of the global minimizer where no other local minimia is present, gradual shrinkage of $\delta$ (smooth transition between problems not to be prone to new local minima), and continuity of $f_\delta$ lead to finding the global minimizer. Rigid mathematical analysis on $\delta$ shrinkage rate is provided in Section \ref{TA}.
\begin{prop}\label{prop2}
When $\delta \to \infty$, $n-F_\delta(\boldsymbol{X})\to \frac{\normF{\boldsymbol{X}}^2}{2\delta^2}$
\end{prop}
\begin{proof}
When $\delta \to \infty$, the Taylor expansion for $f_\delta(x)$ is as follows:
\begin{equation*}
\nonumber \exp{(-\frac{x^2}{2\delta^2})} = 1-\frac{x^2}{2\delta^2}+\mathcal{O}(\frac{1}{\delta^4})
\end{equation*}
\begin{align}
\nonumber F_\delta(\boldsymbol{X})=\sum_{i=1}^{n}\exp{(\frac{-\sigma_i^2(\boldsymbol{X})}{2\delta^2})}=\sum_{i=1}^n{(1-\frac{\sigma_i^2(\boldsymbol{X})}{2\delta^2}})+\mathcal{O}(\frac{1}{\delta^4})
\end{align}
\begin{align}
\nonumber \Rightarrow n-F_\delta(\boldsymbol{X})=\frac{1}{2\delta^2}\sum_{i=1}^{n}\sigma_i^2(\boldsymbol{X})+\mathcal{O}(\frac{1}{\delta^4})=\frac{\normF{\boldsymbol{X}}^2}{2\delta^2} +\mathcal{O}(\frac{1}{\delta^4})
\end{align}
\begin{equation}
\Rightarrow n-F_\delta(\boldsymbol{X})\to \frac{\normF{\boldsymbol{X}}^2}{2\delta^2}
\end{equation}
\end{proof}
Hence, when $\delta \to \infty$, the objective function in problem \ref{p3} tends to
$
\frac{\normF{\boldsymbol{X}}^2}{2\delta^2}+\lambda
\mathcal{H}_\Omega(\boldsymbol{X})
$, which is strictly convex and GD can be applied to find its global minimizer. Next, GNC is leveraged until the global minimizer is reached. Algorithm \ref{Algorithm 1} in the subsequent section, includes the detailed procedure.
\section{The Proposed Algorithm}\label{proposed}
Suppose $\boldsymbol{X}$ has the SVD $\boldsymbol{X}=\mathbf{U}diag({\sigma(\boldsymbol{X}}))\mathbf{V}^T$, where $\sigma(\boldsymbol{X})=[\sigma_1(\boldsymbol{X}),...,\sigma_n(\boldsymbol{X})]^T$. It is shown in \cite{malek2014recovery} that  $G_\delta(\boldsymbol{X}):=\frac{\partial F_\delta(\boldsymbol{X})}{\partial \boldsymbol{X}}$ (gradient of $F_\delta(\boldsymbol{X})$) can be obtained as:
	\begin{equation}
	G_\delta(\boldsymbol{X})=\mathbf{U}diag\{-\frac{\sigma_1}{\delta^2}\exp(-\frac{\sigma_1^2}{2\delta^2}), ... ,-\frac{\sigma_n}{\delta^2}\exp(-\frac{\sigma_n^2}{2\delta^2})\}\mathbf{V}^T,
	\end{equation}
The derivative of the uni-variate Huber loss for entries in $\Omega$ can be calculated as follows:
  \[
    {\tilde{\mathcal{H}}_{ij}}^\prime(x_{ij})=\left\{
                \begin{array}{ll}
             -g, \quad \quad \quad \quad \quad \quad x_{ij}-m_{ij} \leq -\frac{g}{2}\\
                  2(x_{ij}-m_{ij}), \quad \quad |x_{ij}-m_{ij}| \leq \frac{g}{2}\\
                  g, \quad \quad \quad \quad \quad \quad \quad x_{ij}-m_{ij} \geq \frac{g}{2}\\ 
                \end{array}
              \right.
  \]
Let $G_H(\boldsymbol{X})$ denote $\frac{\partial (H_\Omega(\boldsymbol{X}))}{\partial \boldsymbol{X}}$. We have:
  \[
  G_{H_{ij}}(\boldsymbol{X})
    =\left\{
                \begin{array}{ll}
               {\tilde{\mathcal{H}}_{ij}}^\prime(x_{ij}) , \quad \quad \quad \quad \quad \quad (i,j)\in\Omega\\
                  0, \quad \quad \quad \quad \quad \quad \quad \quad \quad \quad o.w.
                \end{array}
              \right.
  \]
The gradient of $\tilde{G}_\delta(\boldsymbol{X},\lambda)$ is therefore given as: 
\begin{equation}
\nabla_{\tilde{G}_\delta}(\boldsymbol{X},\lambda)=-G_\delta(\boldsymbol{X})+\lambda G_H(\boldsymbol{X})
\end{equation}
Finally, taking into account the GNC procedure and the fact that gradient look-up table is available for $\nabla_{\tilde{G}_\delta}(\boldsymbol{X},\lambda)$, our proposed algorithm \textbf{QMC-HANDS} is given in Algorithm \ref{Algorithm 1}. The convergence criteria in Algorithm \ref{Algorithm 1} are based on relative difference in the Frobenius norm of consecutive updates.
\begin{algorithm}[h!] 
	\begin{algorithmic}[1]\caption {The Proposed Method for QMC Using \textbf{H}uber Loss \textbf{and S}RF: \textbf{QMC-HANDS}}\label{Algorithm 1}
	\color{black}
		\State \textbf{Input}:
		\State Observation matrix $\boldsymbol{M}$, the set of observed indices $\Omega$.  
		\State The quantization levels $m_{ij}$, the quantization gap $\frac{g}{2}$, the quantization lower and upper bounds $u_{ij}, l_{ij}$.
		\State The gradient step size $\mu$, the $\delta$ decay factor $\alpha$, the regularization factor $\lambda$, the $\delta$ initiation constant $C$.
\State \textbf{Output}:
\State The recovered matrix $\boldsymbol{X}^*$.
\State \textbf{procedure QMC-HANDS}:
\State $k\gets 0$
\State $\delta \gets C\sigma_{max}({\boldsymbol{M}})$
\State $\boldsymbol{Z}^0 \gets  \text{arg}\underset{\boldsymbol{X}}{\min} \frac{\normF{\boldsymbol{X}}^2}{2\delta^2}+\lambda
\mathcal{H}_\Omega(\boldsymbol{X})$
\While {not converged}
\State $\boldsymbol{X}^0 \gets
 \boldsymbol{Z}^k$
\State $k \gets k+1$
\State $i \gets 0$
\While {not converged}
\State $G^i \gets -G_\delta(\boldsymbol{X}^i)+\lambda G_\mathcal{H}(\boldsymbol{X}^i)$
\State $\boldsymbol{X}^{i+1} \gets \boldsymbol{X}^{i}-\mu G^i$
\State $i \gets i+1$
\EndWhile
\State $\boldsymbol{Z}^k \gets \boldsymbol{X}^i$
		\State $\delta \gets \delta \alpha$
\EndWhile\\
		\Return $\boldsymbol{X}^* \leftarrow \boldsymbol{Z}^k$
		\State \textbf{end procedure}
	\end{algorithmic}
\end{algorithm}
\section{Theoretical Analysis on Global Convergence}\label{TA}
In this section, we propose a sufficient decrease condition for $\delta$ which ensures the GD is not trapped in local minima, and the global minimizer is achieved. Suppose a warm-start ${\boldsymbol{X}^i}$ holds in $\normF{{\boldsymbol{X}^i}-\boldsymbol{X}^*} = \epsilon$ for a positive constant $\epsilon$, and $\tilde{G}_\delta$ is convex on $\mathcal{B}_\epsilon = \{\boldsymbol{X}| \normF{\boldsymbol{X}-\boldsymbol{X}^*} \leq \epsilon \}$. Let $D_{\delta}(\boldsymbol{X})$ and $D_H(\boldsymbol{X})$ denote the Hessian of SRF and Huber, respectively. We have $-D_\delta(\boldsymbol{X})+\lambda D_H(\boldsymbol{X}) \succcurlyeq \mathbf{0}$ on $\mathcal{B}_\epsilon$. $\normF{D_\delta}$ is $\mathcal{O}(\frac{1}{\delta^3})$ since $G_\delta$ is $\mathcal{O}(\frac{1}{\delta^2})$. Let $T_1({\boldsymbol{X}})=\delta^3D_\delta(\boldsymbol{X})$ (normalized w.r.t $\delta$). Thus, $\forall \boldsymbol{X} \in \mathcal{B}_\epsilon$, we have: 
\begin{equation}
-\frac{1}{\delta^3}T_1({\boldsymbol{X}})+\lambda D_H(\boldsymbol{X}) \succcurlyeq \mathbf{0} \Rightarrow \delta^3 \lambda D_H(\boldsymbol{X}) \succcurlyeq T_1(\boldsymbol{X})
\end{equation}
This gives a lower bound for $\delta$ in the $i$-th iteration as $\delta^i=arg\underset{\delta}{\min}~ \{\forall \boldsymbol{X} \in \mathcal{B}_\epsilon: \delta^3\lambda D_H(\boldsymbol{X}) - T_1(\boldsymbol{X}) \succcurlyeq \mathbf{0}\}$.
% Now, let $M_\delta(\epsilon) =\underset{\boldsymbol{X} \in \mathcal{B}_\epsilon}\max~\normF{ \nabla_{\tilde{G}_\delta}(\boldsymbol{X},\lambda)}$. 
%and $m_\delta(\epsilon)$ be defined as: 
%\begin{equation}
%m_\delta(\epsilon)=\underset{(\boldsymbol{X},\boldsymbol{Y}) \in \mathcal{B}_\epsilon}{\min}~\frac{\tilde{G}_\delta(\boldsymbol{Y},\lambda)-\tilde{G}_\delta(\boldsymbol{X},\lambda)}{||\boldsymbol{X}-\boldsymbol{Y}||_F}
%\end{equation}
Assume the problem with warm-start ${\boldsymbol{X}^i}$ is optimized on $\mathcal{B}_\epsilon$ to reach at a new minimizer $\boldsymbol{X}^{i+1}$. By assumption, $\boldsymbol{X}^{i+1}$ is closer to $\boldsymbol{X}^*$ (in Frobenius norm) than $\boldsymbol{X}^i$ since the smaller $\delta$, the better rank approximation. Suppose $\normF{\boldsymbol{X}^{i+1}-\boldsymbol{X}^*}=\epsilon(1-r)$. Now, let $\delta^{i+1}=arg\underset{\delta}{\min}~ \{\forall \boldsymbol{X} \in \mathcal{B}_{\epsilon(1-r)}: \delta^3\lambda D_H(\boldsymbol{X}) - T_1(\boldsymbol{X}) \succcurlyeq \mathbf{0}\}$. By definition, $\mathcal{B}_{\epsilon(1-r)} \subseteq {B}_{\epsilon}$. Therefore, $\delta^{i+1} < \delta^{i}$ by the definition of the minimizer. The sufficient decrease ratio is given as follows: $\alpha^i(r)=\frac{\delta^{i+1}}{\delta^{i}}$, which depends on $r$.\\
\section{Numerical Experiments}\label{NE}
In this section, we provide numerical experiments conducted on the \textit{MovieLens100K} dataset \cite{harper2016movielens} and \cite{Movielens}. \textit{MovieLens100K} contains $100,000$ ratings (instances) $(1-5)$ from $943$ users on $1682$ movies, where each user has rated at least $20$ movies. Our purpose is to predict the ratings which have not been recorded or completed by users.
We assume this rating matrix is a quantized version of a genuine low-rank matrix and recover it using our algorithm. Then, a final quantization can be applied to predict the missing ratings. We compare the learning accuracy and the computational complexity of our proposed approach to state-of-the-art methods discussed in introduction on \textit{MovieLens100K} \ref{Intro}: 
\textbf{Logarithmic Barrier Gradient Method (LBG)} \cite{bhaskar}, \textbf{SPARFA-Lite} (abbreviated as \textbf{SL}) \cite{lan2014matrix}, \cite{lan2014quantized}, and \textbf{QMC-BIF} \cite{esmaeili2018recovering}  in Tables \ref{T1}, \ref{T2}.
\begin{table}[h]
\caption{Recovery RMSE for Different Methods on MovieLens100K}\label{T1}
\begin{center}
\begin{tabular}{ |c|c|c|c|c| }
\hline
 MR & \textbf{SL} & \textbf{LBG} & \textbf{QMC-BIF} & \textbf{QMC-HANDS}\\ 
 \hline
 $10\%$ & $1.316$ & $1.180$ & $0.943$ & $0.898$\\ 
 \hline
 $20\%$ &  $1.825$ & $1.703$ & $1.375$ & $1.287$\\ 
 \hline
 $30\%$ &  $2.608$ & $2.459$ & $2.017$ & $1.773$\\  
 \hline
 $50\%$ &  $3.712$ & $3.521$ & $3.116$ & $2.699$\\  
 \hline
\end{tabular}
\end{center}
\end{table}
\begin{table}[h]
\caption{Computational Runtimes of Different Methods on MovieLens100K (in Seconds)}\label{T2}
\begin{center}
\begin{tabular}{ |c|c|c|c|c| }
\hline
 MR & \textbf{SL} & \textbf{LBG} & \textbf{QMC-BIF} & \textbf{QMC-HANDS}\\ 
 \hline
 $10\%$ & $635$ & $573$ & $612$ & $426$\\ 
 \hline
 $20\%$ &  $647$ & $582$ & $726$ & $430$\\ 
 \hline
 $30\%$ &  $672$ & $593$ & $792$ & $445$\\  
\hline
 $50\%$ &  $718$ & $601$ & $843$ & $482$\\  
\hline
\end{tabular}
\end{center}
\end{table}
We abbreviate the term "missing rate" in Tables \ref{T1}, \ref{T2} with MR. We have induced different MR percentages on the \textit{MovieLens100K} dataset and averaged the performance of each method over $20$ runs of simulation. $\alpha, \mu, C, \lambda$ are set using cross-validation. \\
It can be seen that owing to the differentiability and smoothness, our proposed method is fast. As it can be found in Table \ref{T2}, the computational runtime of the \textbf{QMC-HANDS} is reduced in some cases by $20\%-25\%$ compared to the state-of-the-art. The computational time in seconds are measured on an $@$Intel Core i7 6700 HQ 16 GB
RAM system using MATLAB \textsuperscript{\tiny\textregistered}. In addition, the learning accuracy of our method is enhanced up to $15\%$ in the best case, and outperforms other mentioned methods in the remaining simulation scenarios as reported in Table \ref{T1}. The superiority of our proposed algorithm compared to other mentioned methods is also observed on \textbf{synthetic} datasets. We aim to include the results on synthesized data in an extended work as the future work of this paper.
It is needless to say that like \cite{esmaeili2018recovering}, no projection is required in our proposed algorithm in contrast to \cite{bhaskar}, and \cite{lan2014matrix}.
\section{Conclusion}\label{conc}
In this paper, a novel approach to Quantized Matrix Completion (QMC) using Huber loss measure is introduced. A novel algorithm, which is not restricted to have initial rank knowledge, is proposed for an unconstrained differentiable optimization problem. We have established rigid and novel theoretical analyses and convergence guarantees for the proposed method. The experimental contribution of our work includes enhanced accuracy in recovery (up to $15\%$), and noticeable computational complexity reduction ($20\%-25\%$) compared to state-of-the-art methods as illustrated in numerical experiments.
\bibliographystyle{IEEEtran} 
\bibliography{refQMC}
\end{document}